\documentclass{article}
\usepackage[titletoc,title]{appendix}
\usepackage{times}
\usepackage{graphicx} 
\usepackage{subfigure} 

\usepackage{natbib}

\usepackage{url}

\usepackage{algorithm}
\usepackage{algorithmic}
\usepackage[fleqn]{amsmath}
\usepackage{amsthm}
\usepackage{amssymb}
\newtheorem{theorem}{Theorem}
\newtheorem{corollary}{Corollary}
\newtheorem{lemma}{Lemma}

\newtheorem{remark}{Remark}

\newtheorem{assumption}{Assumption}

\DeclareMathOperator{\tr}{tr}
\DeclareMathOperator{\diag}{diag}
\DeclareMathOperator{\var}{Var}
\DeclareMathOperator{\prob}{Pr}
\newcommand{\mv}[1]{\mathbf{#1}}

\usepackage{hyperref}

\usepackage[accepted]{icml2016}

\icmltitlerunning{Why Regularized Auto-Encoders Learn Sparse Representation?}

\begin{document} 

\twocolumn[
\icmltitle{Why Regularized Auto-Encoders Learn Sparse Representation?}

\icmlauthor{Devansh Arpit}{devansha@buffalo.edu}
\icmlauthor{Yingbo Zhou}{yingbozh@buffalo.edu}
\icmlauthor{Hung Q. Ngo}{hungngo@buffalo.edu}
\icmlauthor{Venu Govindaraju}{govind@buffalo.edu}
\icmladdress{SUNY Buffalo}

\icmlkeywords{auto-encoders, sparse representation}

\vskip 0.3in
]

\addtolength{\baselineskip}{-0.15mm}

\begin{abstract}
	Sparse distributed representation is the key to learning useful features in deep
	learning algorithms, because not only it is an efficient mode of data
	representation, but also -- more importantly -- it captures the generation
	process of most real world data. While a number of regularized auto-encoders
	(AE) enforce sparsity explicitly in their learned representation and others don't, there has been little formal analysis on what encourages sparsity in these
	models in general. Our objective is to formally study this general problem for
	regularized auto-encoders. We provide sufficient conditions on both regularization
	and activation functions that encourage sparsity. We show that multiple popular
	models (de-noising and contractive auto encoders, e.g.) and activations
	(rectified linear and sigmoid, e.g.) satisfy these conditions; thus, our
	conditions help explain sparsity in their learned representation. Thus our
	theoretical and empirical analysis together shed light on the properties of
	regularization/activation that are conductive to sparsity and unify
	a number of existing auto-encoder models and activation functions under the same
	analytical framework. 
\end{abstract}

\section{Introduction}
Sparse Distributed Representation (SDR) \cite{hinton1984distributed} constitutes
a fundamental reason behind the success of deep learning. On one hand, it is an
efficient way of representing data that is robust to noise; in fact, some of the main advantages of sparse distributed representation in the context of deep neural networks has been shown to be information disentangling and manifold flattening \cite{better_mixing}, as well as better linear separability and representational power \cite{sparse_rectifier_net}. On the other hand, and more importantly, SDR captures the data generation process itself and is biologically inspired \cite{sr_hubel1959,sr_brain4,ann_brain3}, which makes this mode of representation useful in the first place.

For these reasons, our objective in this paper is to investigate why a number of
regularized Auto-Encoders (AE) exhibit similar behaviour, especially in terms of
learning sparse representations. AEs are especially interesting for this matter
because of the clear distinction between their learned \textit{encoder}
representation and \textit{decoder} output. This is in contrast with other deep
models where there is no clear distinction between the encoder and decoder
parts. The idea of AEs learning sparse representations (SR) is not new. Due to
the aforementioned biological connection between SR and NNs, a natural follow-up
pursued by a number of researchers was to propose AE variants that encouraged
sparsity in their learned representation \cite{Lee08sparsedeep,psd,sparse_ae}.
On the other hand, there has also been work on empirically analyzing/suggesting
the sparseness of hidden representations learned after pre-training with
unsupervised models \cite{ZeroBiasAutoencoders,rbm_sc,relu_rbm}. However, to the
best of our knowledge, there has been no prior work formally analyzing why
regularized AEs learn sparse representation in general. The main challenge
behind doing so is the analysis of non-convex objective functions. In addition,
questions regarding the efficacy of activation functions and the choice of
regularization on AE objective are often raised since there are multiple
available choices for both. We also try to address these questions with regards
to SR in this paper.

We address these questions in two parts. First, we prove sufficient conditions
on AE regularizations that encourage low pre-activations in hidden units. We
then analyze the properties of activation functions that when coupled with such
regularizations result in sparse representation. Multiple popular activations
have these desirable properties. Second, we show that multiple popular AE
objectives including {\em de-noising auto-encoder} (DAE)
\cite{Vincent:2008:ECR:1390156.1390294} and {\em contractive auto-encoder} (CAE)
\cite{Rifaicae} indeed have the suggested form of regularization; thus
explaining why existing AEs encourage sparsity in their latent representation.
Based on our theoretical analysis, we also empirically study multiple popular AE
models and activation functions in order to analyze their comparative behaviour
in terms of sparsity in the learned representations.  Our analysis thus shows
why various AE models and activations lead to sparsity. As a result, they are
unified under a framework uncovering the fundamental properties of regularizations and activation functions that most of these existing models possess.

\section{Auto-Encoders and Sparse Representation}

Auto-Encoders (AE) \cite{backprop, autoencoder} are a class of single hidden
layer neural networks trained in an unsupervised manner. It consists of an {\em
	encoder} and a {\em decoder}. An input $({\mathbf{x}} \in \mathbb{R}^{n})$ is
first mapped to the latent space with 
$\mathbf h = f_{e}(\mathbf{x}) = s_{e}(\mathbf{W}\mathbf{x} + \mathbf b_e)$
is the hidden representation vector, $s_{e}$ is the encoder
activation, $\mathbf{W} \in \mathbb{R}^{m \times n}$ is the weight
matrix, and $\mathbf b_{e} \in \mathbb{R}^{m}$ is the encoder bias. 
Then, it maps the hidden output back to the original space by 
$\mathbf{y} = f_{d}(\mathbf{h}) = s_{d}(\mathbf{W}^{T}\mathbf{h}) $
where $\mathbf{y}$ is the reconstructed counterpart of $\mathbf{x}$ and  $s_{d}$ is
the decoder activation. The objective of a basic auto-encoder is to minimize the
following with respect to the parameters $\{ \mathbf{W},\mathbf{b}_{e} \}$
%
\begin{equation}
\mathcal{J}_{AE} = 
\mathbb{E}_{\mathbf{x}} [ \ell(\mathbf{x},f_{d}(f_{e}({\mathbf{x}})))]
\end{equation}
where $\ell(\cdot)$ is the squared loss function. The motivation behind this
objective is to capture predominant repeating patterns in data. Thus although
the auto-encoder optimization learns to map an input back to itself, the focus
is on learning a noise invariant representation (manifold) of data.

\subsection{Part I: What encourages \textit{sparsity} during Auto-Encoder training?}
\label{sec_sparsity_reason}

\subsubsection{Sparsity and our assumption}

Learning a dictionary adapted to a set of training data such that the latent code is sparse is generally formulated as the following optimization problem \cite{sr_brain4}
\begin{equation}
\label{eq_sc}
\min_{\mathbf{W},\mathbf{h}} \quad
\sum_{i=1}^{N} \left( \lVert \mathbf{x}_{i} - {\mathbf{W}}^{T}{\mathbf{h}}_{i}
\rVert^{2} + \lambda \lVert \mathbf{h}_{i} \rVert_{1} \right)
\end{equation}
The above objective is convex in each one of $\mathbf{W}$ and $\mathbf{h}$ when
the other is fixed and hence it is generally solved alternately in each variable
while fixing the other. Note that $\ell_{1}$ penalty is the driving force in the
above objective and forces the latent variable to be sparse. 

This section analyses the factors that are required for sparsity in AEs. Note
that in~\eqref{eq_sc} we optimize for a different parameter $\mathbf{h}_{i} $
for each corresponding sample. In the case of AEs, we do not have a separate
parameter that denotes the hidden representation corresponding to every sample
individually. Instead the hidden representation for every sample is a function
of the sample itself along with other network parameters. So in order to define
the notion of sparsity of hidden representation in AEs, we will treat each
hidden unit $h_{i} = s_{e}( \mathbf{{W}}_{j}\mathbf{x} + b_{e_{j}} )$ as a
random variable which itself is a function of the random variable $\mathbf{x}$.
Then the average activation fraction of a unit is the (probability) mass of (data) distribution for which the hidden unit activates. For finite sample datasets, this becomes the fraction of data samples for which the unit activates.

Also note that SDR dictates that all representational units participate in data
representation while very few units activate for a single data sample. Thus a
major difference between SDR and SR is that of {\em dead units} (units that do
not activate for any data sample) since sparsity can in general also be achieved
when most units are dead. However, the latter scenario is undesirable because it
does not truly capture SDR. Thus we model and study the conditions that
encourage sparsity in hidden units; and we also empirically show these
conditions are capable of achieving SDR.

For our analysis, we will use linear decoding which addresses the case of continuous real valued data distributions. We will now show that both regularization and activation function play an important role for achieving sparsity. In order to do so, we make the following assumption,
\begin{assumption}
	\label{assump_residual}
	We assume that the data $\mv x$ is drawn from a distribution $\mathbf{x} \sim
	\mathcal{X}$ for which $\mathbb{E}_{\mathbf{x}}[\mathbf{x}] = \mathbf{0} $ and $
	\mathbb{E}_{\mathbf{x}}[ \mathbf{x}\mathbf{x}^{T} ] = \mathbf{I} $ where $
	\mathbf{I} $ is the identity matrix.
	
	Further, let $ \mathbf{r}^{t}_{\mathbf{x}} \triangleq \mathbf{x} - \mathbf{W}^{T}f_{e}({\mathbf{x}})$ denote the reconstruction residual during auto-encoder training at any iteration $ t $ for training sample $ \mathbf{x} $. Then we assume every dimension of $ \mathbf{r}_{\mathbf{x}}^{t} $ is \textit{i.i.d.} random variable following a Gaussian distribution with mean $ 0 $ and standard deviation $ \sigma_{r} $.
\end{assumption}

Before proceeding, first we establish an important condition needed by AEs
for exhibiting sparse behaviour. Consider the pre-activation of an AE
\begin{equation}
{a}_{j}^{t} = \mathbf{{W}}_{j}^{t}\mathbf{x} + {b}_{e_{j}}^{t}
\end{equation}
Here $ j $ and $ t $ denote the $ j^{th} $ hidden unit and $ t^{th} $ training
iteration respectively, and $ \mathbf{{W}}_{j}^{t} $ denotes the $ j^{th} $ row
of $ \mathbf{W} $. Then notice when Assumption \ref{assump_residual} is true, if
we remove the encoding bias from the AE optimization, the expected
pre-activation becomes $\mathbb{E}_{\mathbf{x}}[{a}_{j}^{t}] =
\mathbb{E}_{\mathbf{x}}[\mathbf{{W}}_{j}^{t}\mathbf{x} ] = 0 $ unconditionally
for all iterations. Consider any activation function $ s_{e}(.) $ with
activation threshold $ \delta_{\min} $, i.e. any data sample with $j^{\text{th}}$
pre-activation $ a_{j}^{t} $ would de-activate the unit if $ a_{j}^{t} <=
\delta_{\min}$ and activate it otherwise. Then the only way for a unit to
exhibit sparse behaviour (over a data distribution) when the expected
pre-activation is always zero, is for the majority of the samples to have 
pre-activation below $\delta_{\min}$. Then, in order
for the average to be zero, the minority above the threshold will have taken
larger values on average compared to the majority. However, this strategy limits
the degree of sparsity that a unit can achieve for any given data distribution
following Assumption \ref{assump_residual}, when the weight lengths are upper
bounded because the pre-activation value also become upper bounded. The bounded
weight length condition is desired in practice for convergence and is achieved
by regularizations like weight decay and Max-Norm \cite{Hinton12Dropout}. Thus,
in order for hidden units to exhibit sparse behaviour, 
encoding bias needs to be a part of AE optimization. 

Having established the importance of encoding bias, we make the following
deduction based on the above assumption,

\begin{lemma}
	\label{lemma_loss_grad_bound}
	If assumption \ref{assump_residual} is true, and encoding activation function $ s_{e}(.) $ has first derivative in $ [0,1] $, then $ {\partial \mathcal{J}_{AE}}/{\partial {b}_{e_{j}}} \in \left[ -2\sigma_{r} \sqrt{n}\lVert \mathbf{W}_{j} \rVert, 2\sigma_{r} \sqrt{n}\lVert \mathbf{W}_{j} \rVert  \right]$.
\end{lemma}
Using the above result , the theorem below gives a \textit{sufficient} condition
on regularization functions needed for forcing the average pre-activation value
($ \mathbb{E}[( {a}_{j}^{t} )] $) to keep on reducing after every training iteration. 

\begin{theorem}
	\label{th_ae_reg_form}
	Let $\{ \mathbf{{W}}^{t} \in \mathbb{R}^{m \times n}, \mathbf{{b}}_{e}^{t} \in \mathbb{R}^{m}  \}$ be the parameters of a regularized auto-encoder ($ \lambda > 0 $)
	\begin{equation}
	\mathcal{J}_{RAE} = \mathcal{J}_{AE} + \lambda \mathcal{R}({\mathbf{W}},{\mathbf{b}}_{e})
	\end{equation}
	
	at training iteration $t$ with regularization term $\mathcal{R}({\mathbf{W}},{\mathbf{b}}_{e})$, activation function $s_{e}(.)$ and define pre-activation ${a}_{j}^{t} = \mathbf{{W}}_{j}^{t}\mathbf{x} + {b}_{e_{j}}^{t}$ (thus ${h}_{j}^{t} = s_{e}({a}_{j}^{t})$). \textbf{If} $\lambda \frac{\partial \mathcal{R}}{\partial {b}_{e_{j}}} > 2 \sigma_{r} \sqrt{n} \lVert \mathbf{W}_{j} \rVert $, where $j \in \{ 1,2, \hdots , m \}$, \textbf{then} updating $\{ \mathbf{{W}}^{t}, \mathbf{{b}}_{e}^{t} \}$ along the negative gradient of $\mathcal{J}_{RAE}$, results in  $\mathbb{E}_{\mathbf{x}}[{a}_{j}^{t+1}] < \mathbb{E}_{\mathbf{x}}[{a}_{j}^{t}]$ \textbf{and} $\var[{{a}_{j}^{t+1}}] = \lVert \mathbf{{W}}_{j}^{t+1} \rVert^{2}$ \textbf{for} all $ t\geq0 $.
\end{theorem}

\paragraph*{Interpretation:} The important thing to notice in the above theorem is that larger values of $ \lambda $ is expected to lead to lower expected pre-activation values since,
\begin{equation}
\mathbb{E}_{\mathbf{x}} \left[ {a}_{j}^{t+1} \right] = \mathbb{E}_{\mathbf{x}} \left[ {a}_{j}^{t} \right] - \eta ( \frac{\partial \mathcal{J}_{AE}}{\partial {b}_{e_{j}}} + \lambda  \frac{\partial \mathcal{R}}{\partial {b}_{e_{j}}})
\end{equation}
where $ \eta $ is the learning rate. But this may not be true in general over multiple iterations due to
terms in $ \frac{\partial \mathcal{R}}{\partial {b}_{e_{j}}} $ that depend on
weight vectors that also change every iteration depending on the value of $
\lambda $. However, we are generally interested in the direction of the weight
vectors during reconstruction instead of their scale. Thus if we fix the length
of weight vectors (to say, unit length), then the term $ \frac{\partial
	\mathcal{R}}{\partial {b}_{e_{j}}} $ will be bounded by a fixed term w.r.t. weight vectors and
will only depend on the bias and data distribution. Under these circumstances,
increasing the value of $ \lambda $ is conducive to lower expected
pre-activation if $ \frac{\partial \mathcal{R}}{\partial {b}_{e_{j}}} $ is
strictly greater than zero. On the other hand, if $ \frac{\partial
	\mathcal{R}}{\partial {b}_{e_{j}}} = 0$, then changing the value of $ \lambda $
should not have significant effect on expected pre-activation values, especially when
the weight length is fixed. In the case when the weight length is not fixed,
changing the value of $ \lambda $ will affect the value of weight length, which
in turn will affect the term $ \frac{\partial \mathcal{J}_{AE}}{\partial
	{b}_{e_{j}}} $ which also affects expected pre-activation of a unit; but this
effect is largely unpredictable depending on the form of $ \frac{\partial
	\mathcal{J}_{AE}}{\partial {b}_{e_{j}}} $. In the next section, we will connect
the notions of expected pre-activation and sparsity, for activation functions
with certain properties which will extend the above arguments to the sparsity of
hidden units.

Finally, in the relaxed cases when weight lengths are not constrained to have a
fixed length, an upper bound on weight vectors' length can easily be guaranteed
using \textit{Max-norm Regularization} or \textit{Weight Decay} which are widely
used tricks while training deep networks \cite{Hinton12Dropout}. In the prior
case every weight vector is simply constrained to lie within an $\ell_{2}$ ball
($\lVert \mathbf{W}_{j} \rVert_{2} \leq c$ $\forall j \in [m]$, where $c$ is a
fixed constant) after every gradient update.

Having shown the property of regularization functions that encourages lower
pre-activations, we now introduce two classes of regularization functions that
inherit this property and thus manifest the predictions made above.


\begin{corollary}
	\label{cor_ae_reg_form2}
	\textbf{If} $ s_{e}$ is a non-decreasing activation function with first derivative in $ [0,1] $ \textbf{and} $\mathcal{R}= \sum_{j=1}^{m} f(\mathbb{E}_{\mathbf{x}}[{h}_{j}])$ for any monotonically increasing function $ f(.) $, \textbf{then} $ \exists \lambda>0 $ such that updating $\{ \mathbf{{W}}^{t}, \mathbf{{b}}_{e}^{t} \}$ along the negative gradient of $ \mathcal{J}_{RAE} $ results in  $\mathbb{E}_{\mathbf{x}}[{a}_{j}^{t+1}] \leq \mathbb{E}_{\mathbf{x}}[{a}_{j}^{t}]$ \textbf{and} $\var[{{a}_{j}^{t+1}}] = \lVert \mathbf{{W}}_{j}^{t+1} \rVert^{2}$ \textbf{for} all $ t\geq0 $.
\end{corollary}

\begin{corollary}
	\label{cor_ae_reg_form1}
	\textbf{If} $ s_{e} $ is a non-decreasing convex activation function with first derivative in $ [0,1] $ \textbf{and} $\mathcal{R}= \mathbb{E}_{\mathbf{x}} \left[  \sum_{j=1}^{m} \left( \left(\frac{\partial {h}_{j}}{\partial {a}_{j}}\right)^{q}\lVert \mathbf{{W}}_{j}^{t}\rVert_{2}^{p} \right) \right]$, $q \in \mathbb{N}$ , $p \in \mathbb{W}$, \textbf{then} $ \exists \lambda>0 $ such that updating $\{ \mathbf{{W}}^{t}, \mathbf{{b}}_{e}^{t} \}$ along the negative gradient of $ \mathcal{J}_{RAE} $, results in  $\mathbb{E}_{\mathbf{x}}[{a}_{j}^{t+1}] \leq \mathbb{E}_{\mathbf{x}}[{a}_{j}^{t}]$ \textbf{and} $\var[{{a}_{j}^{t+1}}] =  \lVert \mathbf{{W}}_{j}^{t+1} \rVert^{2}$ \textbf{for} all $ t\geq0 $.
\end{corollary}

Above corollaries show that specific regularizations encourage the
pre-activation of every hidden unit in AEs to reduce on average, with
assumptions made only on activation function and the first/second order
statistics of the data distribution. We will show in Section~\ref{sec_part3}
that multiple existing AEs have regularizations of the form above.

\subsubsection{Which activation functions are \textit{good} for Sparse Representation?}

The above analysis in general suggests that non-decreasing convex activation functions encourage lower expected pre-activation for regularization in both corollaries. Also note that a reduction in the expected pre-activation value ($ \mathbb{E}[( {a}_{j}^{t} )] $) does not necessarily imply a reduction in the hidden unit value ($  {h}_{j}^{t}  $) and thus sparsity. However, these regularizations become immediately useful if we consider non-decreasing activation functions with negative saturation at $ 0 $, \textit{i.e.}, $ \lim_{a \rightarrow - \infty} s_{e}(a) = 0$. Now a lower average pre-activation value directly implies higher sparsity!

Before proceeding, we would like to mention that although the general notion of
sparsity in AEs entails majority of units are de-activated, \textit{i.e.}, their value is less than a certain threshold
($ \delta_{\min} $), in practice, a representation that is \textit{truly sparse}
(large number of hard zeros) usually yields better performance
\cite{sparse_rectifier_net,src,sr_img_classification}. 
Extending the argument of theorem \ref{th_ae_reg_form}, we obtain:


\begin{theorem}
	\label{th_activation}
	Let $ p_{j}^{t}$ denote a lower bound of $
		\prob({h}_{j}^{t}\leq \delta_{\min}) $ at iteration $ t $ and $ s_{e}(.) $ be a non-decreasing function with first derivative in $ [0,1] $. \textbf{If} $ \lVert \mathbf{W}_{j}^{t} \rVert_{2} $ is upper bounded independent of $ \lambda $ \textbf{then} $ \exists S \subseteq \mathbb{R}^{+} $ \textbf{and} $ \exists T_{\min},T_{\max} \in \mathbb{N}$ \textbf{such that} $ p_{j}^{t+1}\geq p_{j}^{t} $ $ \forall \lambda \in S $, $T_{\min} \leq t \leq T_{\max} $.
\end{theorem}
The above theorem formally connects the notions of expected pre-activation and expected sparsity of a hidden unit. Specifically, it shows that the usage of non-decreasing activation functions lead to lower expected pre-activation and thus a higher probability of de-activated hidden units when theorem \ref{th_ae_reg_form} applies. This result coupled with the property  $ \lim_{a \rightarrow - \infty} s_{e}(a) = 0$ (de-activated state) implies the average sparsity of hidden units keeps increasing after a sufficient number of iterations ($ T_{\min} $) for such activations. 
Notice that convexity in $ s_{e}(.) $ is only desired for regularizations in corollary \ref{cor_ae_reg_form1}. Thus in summary, non-decreasing convex $ s_{e}(.) $ ensure ${\partial \mathcal{R}}/{\partial {b}_{e_{j}}} $ is positive for regularizations in corollary \ref{cor_ae_reg_form2} and \ref{cor_ae_reg_form1}, which in turn encourages low expected pre-activation for suitable values of $ \lambda $. This finally leads to higher sparsity if $ \lim_{a \rightarrow - \infty} s_{e}(a) = 0$.

Notice we derive the strict inequality ($\mathbb{E}_{\mathbf{x}}[{a}_{j}^{t+1}]
< \mathbb{E}_{\mathbf{x}}[{a}_{j}^{t}]$) in Theorem~\ref{th_ae_reg_form} (and used in Theorem~\ref{th_activation}) even though the corollaries suggest
non-decreasing convex activations imply the relaxed case
($\mathbb{E}_{\mathbf{x}}[{a}_{j}^{t+1}] \leq
\mathbb{E}_{\mathbf{x}}[{a}_{j}^{t}]$). This is done for two reasons: a) ensure
sparsity monotonically increases for iterations $T_{\min} \leq t\leq T_{\max} $,
b) the condition ${\partial \mathcal{R}}/{\partial {b}_{e_{j}}} = 0$ (which
results in $\mathbb{E}_{\mathbf{x}}[{a}_{j}^{t+1}] \leq
\mathbb{E}_{\mathbf{x}}[{a}_{j}^{t}]$) is unlikely for activations with non-zero
first/second derivatives because the term $ \mathcal{R} $ (above corollaries)
depends on the entire data distribution.

The most popular choice of activation functions are ReLU,
Maxout\cite{Goodfellow13maxout}, Sigmoid, Tanh and Softplus. Maxout and Tanh are
not applicable to our framework as they do not satisfy the negative saturation
property. 

\textbf{ReLU: }It is a non-decreasing convex function; thus both corollary \ref{cor_ae_reg_form2} and \ref{cor_ae_reg_form1} apply. Note 
ReLU does not have a second derivative\footnote{\scriptsize In other words, $ \partial^{2} {h}_{j}/\partial {a}_{j}^{2} = \delta(\mathbf{{W}}_{j}\mathbf{x} + {b}_{e_{j}}) $, where $ \delta(.) $ is the Dirac delta function.  Although strictly speaking, $ \partial^{2} {h}_{j}/\partial {a}_{j}^{2} $ is always non-negative, this value is zero everywhere except when the argument is exactly $ 0 $, in which case it is $ +\infty $}. Thus, in practice, this may lead to poor sparsity for the regularization in Corollary~\ref{cor_ae_reg_form1} due to lack of bias gradients from the regularization, i.e. ${\partial \mathcal{R}}/{\partial {b}_{e_{j}}}=0$. On the flip side, the advantage of ReLU is that it enforces hard zeros in the learned representations.

\textbf{Softplus: }It is a non-decreasing convex function and hence encourages sparsity for the suggested AE regularizations. In contrast to ReLU, Softplus has positive bias gradients (hence better sparsity for corollary \ref{cor_ae_reg_form1}) because of its smoothness. On the other hand, note that Softplus does not produce hard zeros due to asymptotic left saturation at $ 0 $.

\textbf{Sigmoid: }Corollary \ref{cor_ae_reg_form2} applies unconditionally to Sigmoid, while corollary \ref{cor_ae_reg_form1} doesn't apply in general. Hence Sigmoid is not guaranteed to lead to sparsity when used with regularizations of form specified in Corollary~\ref{cor_ae_reg_form1}.

Notice all the above activation functions have their first derivative in $ [0,1] $ (a condition required by lemma \ref{lemma_loss_grad_bound}). In conclusion, Maxout and Tanh do not satisfy the negative saturation property at $ 0 $ and hence do not guarantee sparsity, all others-- ReLU, Softplus and Sigmoid-- have properties (at least in principle) that encourage sparsity in learned representations for the suggested regularizations.

\subsection{Part II: Do existing Auto-Encoders learn Sparse Representation?}
\label{sec_part3}
At this point, a natural question to ask is whether existing AEs learn Sparse
Representation. To complete the loop, we show that most of the popular AE
objectives have regularization term similar to what we have proposed in
Corollaries \ref{cor_ae_reg_form2} and \ref{cor_ae_reg_form1} and thus they
indeed learn sparse representation.

\paragraph*{1) De-noising Auto-Encoder (DAE):} DAE
\cite{Vincent:2008:ECR:1390156.1390294} aims at minimizing the reconstruction
error between every sample $\mathbf{x}$ and the reconstructed vector using
its corresponding corrupted version $\tilde{\mathbf{x}}$. The corrupted
version $\tilde{\mathbf{x}}$ is sampled from a conditional distribution
$p(\tilde{\mathbf{x}}_{i}|\mathbf{x}_{i})$. The original DAE objective is
given by
\begin{equation}
\mathcal{J}_{DAE} =  \mathbb{E}_{\mathbf{x}} \left[ \mathbb{E}_{p(\mathbf{\tilde{x}}|\mathbf{x})}[\ell(\mathbf{x},f_{d}(f_{e}(\tilde{\mathbf{x}})))] \right]
\end{equation}
where $p(\mathbf{\tilde{x}}_{i}|\mathbf{x})$ denotes the conditional distribution of $\mathbf{\tilde{x}}$ given $\mathbf{x}$. Since the above objective is analytically intractable due to the corruption process, we take a second order Taylor's approximation of the DAE objective
around the distribution mean $\mu_{\mathbf{x}} = \mathbb{E}_{p(\mathbf{\tilde{x}}|\mathbf{x})}[\mathbf{\tilde{x}}]$ in order to overcome this difficulty,

\begin{theorem}
	\label{th_dae}
	Let $\{ \mathbf{{W}},\mathbf{{b}}_{e}\}$ represent the parameters of a DAE with squared loss, linear decoding, and i.i.d. Gaussian corruption with zero mean and $\sigma^{2}$ variance, at any point of training over data sampled from distribution $\mathcal{D}$. Let ${a}_{j} := \mathbf{{W}}_{j}\mathbf{x} + {b}_{e_{j}}$ so that ${h}_{j} = s_{e}({a}_{j})$ corresponding to sample $\mathbf{x} \sim \mathcal{D}$.
	Then,
	\begin{equation}
	\begin{split}
	\label{eq_dae_theorem}
	\mathcal{J}_{DAE} = \mathcal{J}_{AE} + 
	\sigma^{2} \mathbb{E}_{\mathbf{x}} \left[  \sum_{j=1}^{m} \left( \left(\frac{\partial {h}_{j}}{\partial {a}_{j}}\right)^{2}\lVert \mathbf{{W}}_{j} \rVert_{2}^{4} \right) \right.\\
	\left.+ \sum_{\substack{j,k=1 \\ j\neq k}}^{m}\left( \frac{\partial {h}_{j}}{\partial {a}_{j}} \frac{\partial {h}_{k}}{\partial {a}_{k}} (\mathbf{{W}}_{j}^{T} \mathbf{{W}}_{k})^{2} \right)    \right. \\
	\left. +  \sum_{i=1}^{n}\left(  ({\mathbf{b}_{d}+\mathbf{W}}^{T}{\mathbf{h}} - \mathbf{x})^{T}{\mathbf{W}}^{T} \left( \frac{\partial^{2} {\mathbf{h}} }{\partial {\mathbf{a}}^{2}} \odot {\mathbf{W}}^{i} \odot{\mathbf{W}}^{i} \right) \right) \right]  \\
	+ o(\sigma^{2})
	\end{split}
	\end{equation}
	where $ \frac{\partial^{2} {\mathbf{h}} }{\partial {\mathbf{a}}^{2}}  \in \mathbb{R}^{m}$ is the element-wise $ 2^{nd} $ derivative of $ {\mathbf{h}} $ \textit{w.r.t.} $ {\mathbf{a}} $ and $ \odot $ is element-wise product.
\end{theorem}
The first term of the above regularization is of the form stated in corollary \ref{cor_ae_reg_form1}. Even though the second term doesn't have the exact suggested form, it is straight forward to see that this term generates non-negative bias gradients for non-decreasing convex activation functions (and should have behaviour similar to that predicted in corollary \ref{cor_ae_reg_form1}). Note the last term depends on the reconstruction error which practically becomes small after a few epochs of training and the other two regularization terms take over. Besides, this term is usually ignored as it is not positive-definite. This suggests that DAE is capable of learning sparse representation.


\paragraph*{2) Contractive Auto-Encoder (CAE):} CAE \cite{Rifaicae} objective is given by
\begin{equation}
\mathcal{J}_{CAE} = \mathcal{J}_{AE} + \lambda \mathbb{E}_{\mathbf{x}} \left[  \lVert {J}(\mathbf{x}) \rVert_{F}^{2} \right]
\end{equation}
where ${J}(\mathbf{x}) =  \frac{\partial {{\mathbf{h}}}}{\partial {{\mathbf{x}}}} $ denotes the Jacobian matrix and the objective aims at minimizing the sensitivity of the hidden representation to slight changes in input. 
\begin{remark}
	Let $\{ \mathbf{{W}},\mathbf{{b}}_{e} \}$ represent the parameters of a CAE with regularization coefficient $\lambda$, at any point of training over data sampled from some distribution $\mathcal{D}$. Then,
	\begin{equation}
	\label{eq_cae_theorem}
	\mathcal{J}_{CAE} = \mathcal{J}_{AE} + \lambda \mathbb{E}_{\mathbf{x}} \left[  \sum_{j=1}^{m} \left( \left(\frac{\partial {h}_{j}}{\partial {a}_{j}}\right)^{2}    \lVert \mathbf{{W}}_{j} \rVert_{2}^{2} \right) \right]
	\end{equation}
\end{remark}
Thus CAE regularization also has a form identical to the form suggested in corollary \ref{cor_ae_reg_form1}. Thus the hidden representation learned by CAE should also be sparse. In addition, since the first order regularization term in Higher order CAE (CAE+H) \cite{hcae} is the same as CAE, this suggests that CAE+H objective should have similar properties in term of sparsity. 

\paragraph*{3) Marginalized De-noising Auto-Encoder (mDAE):} mDAE \cite{icml2014c2_cheng14} objective is given by:
\begin{equation}
\label{eq_mdae}
\begin{split}
\mathcal{J}_{mDAE} = \mathcal{J}_{AE}
+ \frac{1}{2} \mathbb{E}_{\mathbf{x}} \left[ \sum_{i=1}^{n} \sigma_{\mathbf{x}i}^{2} \sum_{j=1}^{m}  \frac{\partial^{2} \ell}{\partial {{{h}_{j}}^{2}}} \left( \frac{\partial {{h}_{j}}}{\partial \tilde{{x}}_{i}} \right)^{2} \right]
\end{split}
\end{equation}
where $\sigma_{\mathbf{x}i}^{2}$ denotes the corruption variance intended for the $i^{th}$ input dimension. The authors of mDAE proposed this algorithm with the primary goal of speeding up the training of DAE by deriving an approximate form that omits the need to iterate over a large number of explicitly corrupted instances of every training sample.
\begin{remark}
	Let $\{ \mathbf{{W}},\mathbf{{b}}_{e}\}$ represent the parameters of a mDAE with linear decoding, squared loss and $\sigma_{\mathbf{x}i}^{2} = \lambda$ $\forall i$, at any point of training over data sampled from some distribution $\mathcal{D}$. Then,
	\begin{equation}
	\label{eq_mdae_equiv}
	\mathcal{J}_{mDAE} = \mathcal{J}_{AE} + \lambda \mathbb{E}_{\mathbf{x}} \left[  \sum_{j=1}^{m} \left( \left(\frac{\partial {h}_{j}}{\partial {a}_{j}}\right)^{2} \lVert \mathbf{{W}}_{j} \rVert_{2}^{4} \right) \right]
	\end{equation}
\end{remark}

Apart from justifying sparsity in the above AEs, these equivalences also expose the similarity between DAE, CAE and mDAE regularization as they all follow the form in corollary \ref{cor_ae_reg_form1}. Note how the goal of achieving invariance in hidden and original representation respectively in CAE and mDAE show up as a mere factor of weight length in their regularization in the case of linear decoding.\\
\paragraph*{4) Sparse Auto-Encoder (SAE):} Sparse AEs are given by:
\begin{equation}
\begin{split}
\mathcal{J}_{SAE} = \mathcal{J}_{AE} + \lambda \sum_{j=1}^{m} \left( \rho \log (\rho/{\rho}_{j}) \right. \\
\left. + (1-\rho)\log((1-\rho)/(1-{\rho}_{j})) \right)
\end{split}
\end{equation}
where $ {\rho}_{j} = \mathbb{E}_{\mathbf{x}}[{h}_{j}] $ and $ \rho $ is the desired average activation (typically close to $ 0 $). Thus SAE requires one additional parameter ($ \rho $) that needs to be pre-determined. To make SAE follow our paradigm, we set $ \rho =0 $ and thus tuning the value of $ \lambda $ would automatically enforce a balance between the final level of average sparsity and reconstruction error. Thus the SAE objective becomes
\begin{equation}
\label{eq_sae2}
\mathcal{J}_{SAE} = \mathcal{J}_{AE} - \lambda \sum_{j=1}^{m} \log(1-{\rho}_{j}) \mspace{20mu} (when \mspace{10mu }\rho=0)
\end{equation}
Note for small values of $ {\rho}_{j} $, $ \log (1-{\rho}_{j}) \approx -{\rho}_{j} $. Thus the above objective has a very close resemblance with sparse coding (equation \ref{eq_sc}, except that SC has a non-parametric encoder). On the other hand, the above regularization has a form as specified in corollary \ref{cor_ae_reg_form2} which we have showed enforces sparsity. Thus, although it is expected of the SAE regularization to enforce sparsity from an intuitive standpoint, our results show that it indeed does so from a more theoretical perspective.


\section{Empirical Analysis and Observations}
\label{sec_experiments}
We use the following two datasets for our experiments:

\paragraph*{1. MNIST }\cite{mnistlecun}: It is a $10$ class dataset of handwritten digit images of which $50,000$ images are provided for training.

\paragraph*{2. CIFAR-$10$ }\cite{cifar}: It consists of 60,000 $32 \times 32$ color images of objects in $10$ classes. For CIFAR-$10$, we randomly crop $50,000$ patches of size $8 \times 8$ for training the auto-encoders.

\paragraph*{Experimental Protocols:} Since neural network (NN) optimization is non-convex, training with different optimization conditions (eg. learning rate, data scale and mean, gradient update scheme e.t.c.) can lead to drastically different outcomes. However, one of the very things that make training NNs difficult is well designed optimization strategies without which they do not learn useful features. Our analysis is based on certain assumptions on data distribution and conditions on weight matrices. Thus in order to empirically verify our analysis, we use the following experimental protocols that make the optimization well conditioned.

For all experiments, we use mini-batch stochastic gradient descent with momentum ($ 0.9 $) for optimization, $ 50 $ epochs, batch size $ 50 $ and hidden units $ 1000 $. We train DAE, CAE, mDAE and SAE (using eq. \ref{eq_sae2}) with the same hyper-parameters for all the experiments. For regularization coefficient ($ \sigma^{2} $), we use the values in the set $ \{0,  0.001,0.1^{2},0.2^{2},0.3^{2},0.4^{2},0.5^{2},0.6^{2},0.7^{2},0.8^{2},0.9^{2},$ $1.0 \} $ for all models except DAE where $ \sigma^{2} $ values represent the \textit{variance} of Gaussian noise added. For all models and activation functions, we use squared loss and linear decoding. We initialize the bias to zeros and use normalized initialization \cite{norm_init} for the weights. Further, we subtract mean and divide by standard deviation for all samples. \footnote{\scriptsize We noticed in case of MNIST, it is important to add a large number ($ 0.1 $) to the standard deviation before dividing. We believe this is because MNIST (being binary images with uniform background) does not follow our assumption on data distribution.}

\textit{Learning Rate (LR): }Too small a LR won't move the weights from their initialized region and the convergence would be very slow. On the other hand, if we use too large a learning rate, it will change weight direction very drastically (may diverge), something we don't desire for our predictions to hold. So, we find a middle ground and choose LR in the range $ (0.001,0.005) $ for our experiments.

\paragraph*{Terminology: }We are interested in analysing the sparsity of hidden units as a function of regularization coefficient $ \sigma^{2} $ through out our experiments. Recall that our notion of sparsity \ref{sec_sparsity_reason} is denoted by the fraction of data samples that deactivate a hidden unit instead of the fraction of hidden units that deactivate for a given data sample. This choice was made in order to treat each hidden unit as a random variable. Since we cannot identify a particular hidden unit across auto-encoders trained with different values of $ \sigma^{2} $, the only way for measuring the level of sparsity in auto-encoder units is compute the \textit{Average Activation Fraction}, which is defined as follows:
\begin{equation}
Avg. Act. Fraction = \frac{\sum_{i=1}^{N} \sum_{j=1}^{m} \mathbf{1}(h_{j}^{i} > \delta_{\min}) }{N \times m}
\end{equation} 
Here $ \mathbf{1}(.) $ is the indicator operator, $ h_{j}^{i} $ denotes the $ j^{th} $ hidden unit for the $ i^{th} $ data sample, and $ \delta_{\min} $ is the activation threshold. In the case ReLU, $ \delta_{\min} =0 $, and in the case of Sigmoid and Softplus, $ \delta_{\min} = 0.1 $. Also $ N $ and $ m $ denote the total number of data samples and number of hidden units respectively. Notice sparsity of a hidden unit is inversely related to the average activation fraction for a single unit. Thus our definition of \textit{Avg. Activation Fraction} is the indicator of average sparsity across all hidden units. Finally, while measuring Avg. Activation Fraction during training, we also keep track of fraction of \textit{dead units}. Dead units are those hidden units which deactivate for all data samples and are thus unused by the network for data reconstruction. Notice while achieving sparsity, it is desired that minimal hidden units are dead and all \textit{alive} units activate only for a small fraction of data samples. 

\subsection{Sparsity when Bias Gradient is zero}
One of the main predictions made based on theorem \ref{th_ae_reg_form} is that the sparsity of hidden units should remain unchanged with respect to $ \sigma^{2} $ when the bias gradient $ \frac{\partial \mathcal{R}}{\partial {b}_{e_{j}}} =0 $ and weight lengths are fixed to a pre-determined value because the expected pre-activation becomes completely independent of $ \sigma^{2} $. Notice this prediction only accounts for change in sparsity as a result of change in expected pre-activation of the corresponding unit. Sparsity can also increase when expected pre-activation for that unit is fixed, as a result of change in weight directions such that majority samples take pre-activation values below activation threshold while the minority takes values above it such that the overall expected value remains unchanged. This change in weight directions is also affected by $ \sigma^{2} $ since regularization functions specified in corollary \ref{cor_ae_reg_form1} and \ref{cor_ae_reg_form2} contain both weight and bias terms. However, the latter factor contributing to change in sparsity is unpredictable in terms of changing $ \sigma^{2} $ values. Hence it is desired for sparsity to be largely affected only when bias gradient is present for better predictive power.

Hence we analyse the effect of regularization coefficient ($ \sigma^{2} $) on the sparsity of representations learned by AE models using ReLU activation function with weight lengths constrained to be one. Notice ReLU has zero bias gradient for CAE and mDAE, but also for the equivalent regularization derived for DAE \ref{th_dae}. The plots are shown in figure \ref{tab_zero_bias_grad_w1}.\footnote{\scriptsize For weight length constrained to $ 1 $, CAE and mDAE objectives become equivalent.} 

\begin{figure}[t]
	
	\begin{center}
		\begin{tabular}{  c c  }
			\includegraphics[width=0.45\columnwidth,trim=2.4in 0.1in 3.1in 0.1in,clip]{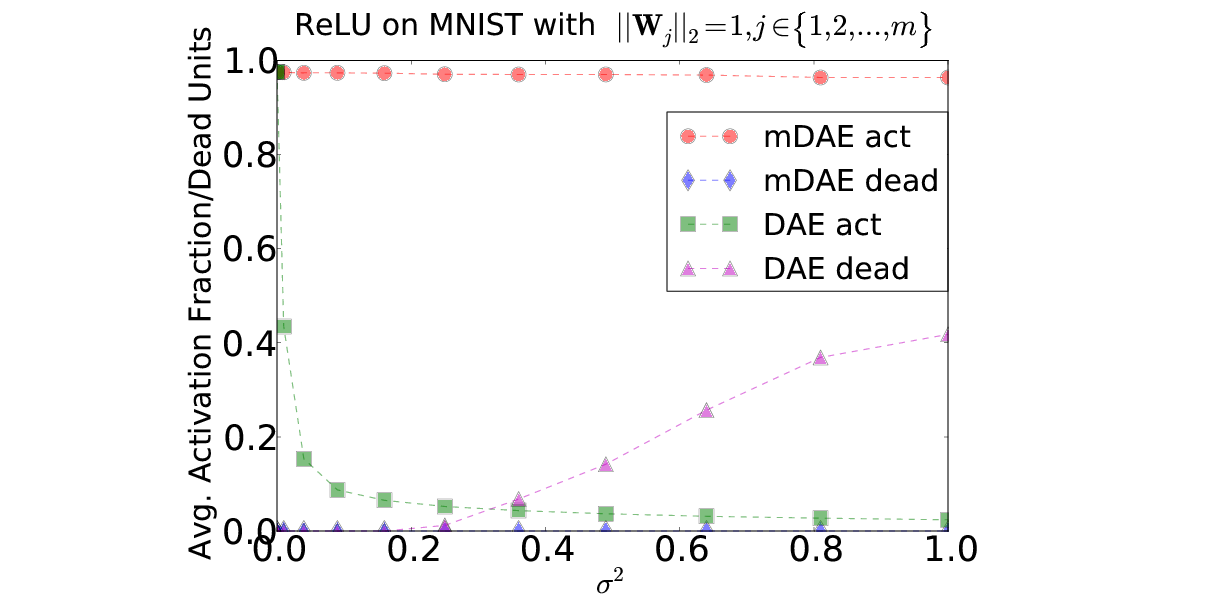}
			& 
			\includegraphics[width=0.45\columnwidth,trim=2.4in 0.1in 3.1in 0.1in,clip]{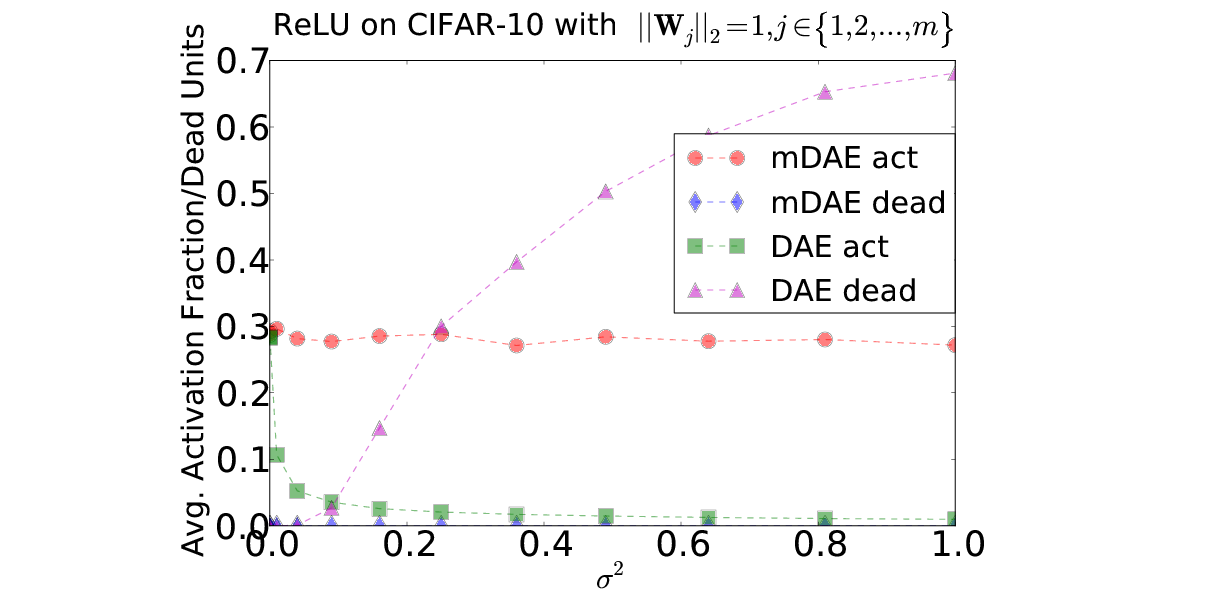}
		\end{tabular}
		\caption{{Trend of average activation fraction \textit{vs.} $ \sigma^{2} $ with weight length constraint using ReLU on MNIST (left) and CIFAR-10 (right).} \label{tab_zero_bias_grad_w1}}
	\end{center}
	\vspace{-20pt}
\end{figure} 

We see that the effect of bias gradient largely dominates the behaviour of hidden units in terms of sparsity. Specifically, as predicted, average activation fraction (and thus sparsity) remains unchanged with respect to regularization coefficient $ \sigma^{2} $ when ReLU is applied to CAE and mDAE due to the absence of bias gradient. 

We also analyse the effect of regularization coefficient ($ \sigma^{2} $) on the sparsity of representations learned by AE models using ReLU activation functions when weight lengths are not constrained. These plots can be seen in fig \ref{tab_zero_bias_grad_w_uncons}. We find that the trend becomes unpredictable for both CAE and mDAE (both datasets have different trends). As discussed after theorem \ref{th_ae_reg_form}, without weight length constraint, $ \sigma^{2} $ affects weight length which in turn affects $ \frac{\partial \mathcal{J}_{AE}}{\partial {b}_{e_{j}}} $ that changes the value of expected pre-activation. However, this effect is unpredictable and thus undesired. 

On the other hand, we see that for DAE, in the constrained length case (fig \ref{tab_zero_bias_grad_w1}), the number of dead units start rising only after the average activation fraction reaches around $ 0.05 $. However, in case of unconstrained weight length, ReLU does not go below the avg. activation fraction of $ 0.1 $. This shows that constrained weight length achieves higher level of sparsity before giving rise to dead units.

In summary, we find that bias gradient dominates the behaviour of hidden units in terms of sparsity. Also, these experiments suggest we get both more predictive power and better sparsity with hidden weights constrained to have fixed (unit) length. Notice this does not restrict the usefulness of the representation leaned by auto-encoders since we are only interested in the filter shapes, and not their scale.

\begin{figure}[!t!b]
	
	\begin{center}
		\begin{tabular}{  c  c  }
			\includegraphics[width=0.45\columnwidth,trim=2.4in 0.1in 3.1in 0.1in,clip]{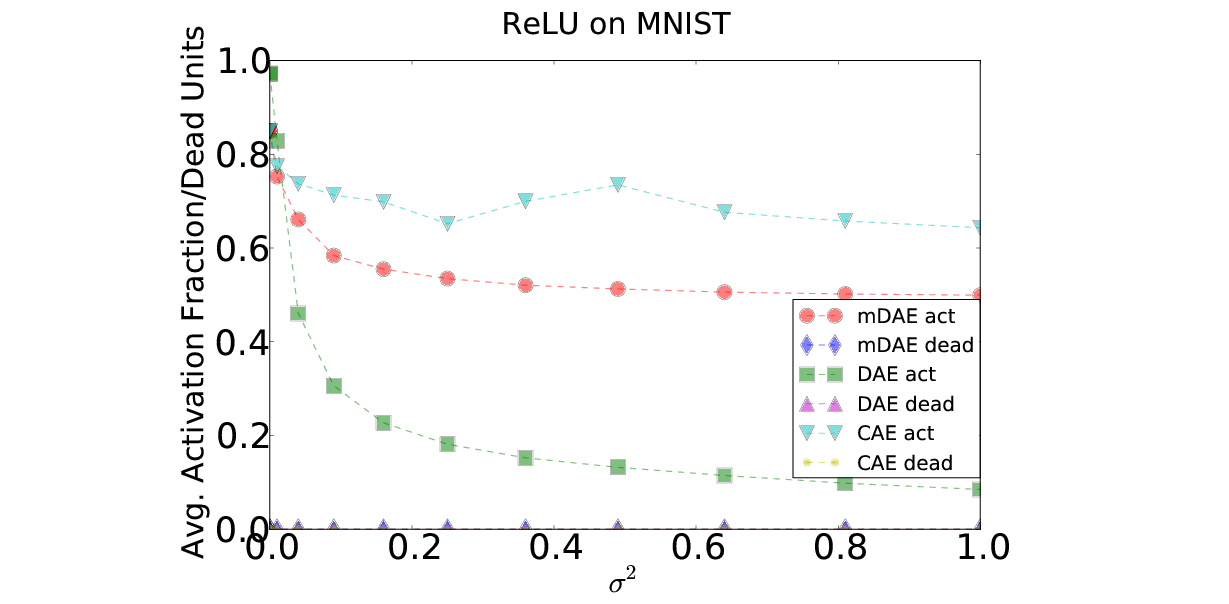}
			& 
			\includegraphics[width=0.45\columnwidth,trim=2.4in 0.1in 3.1in 0.1in,clip]{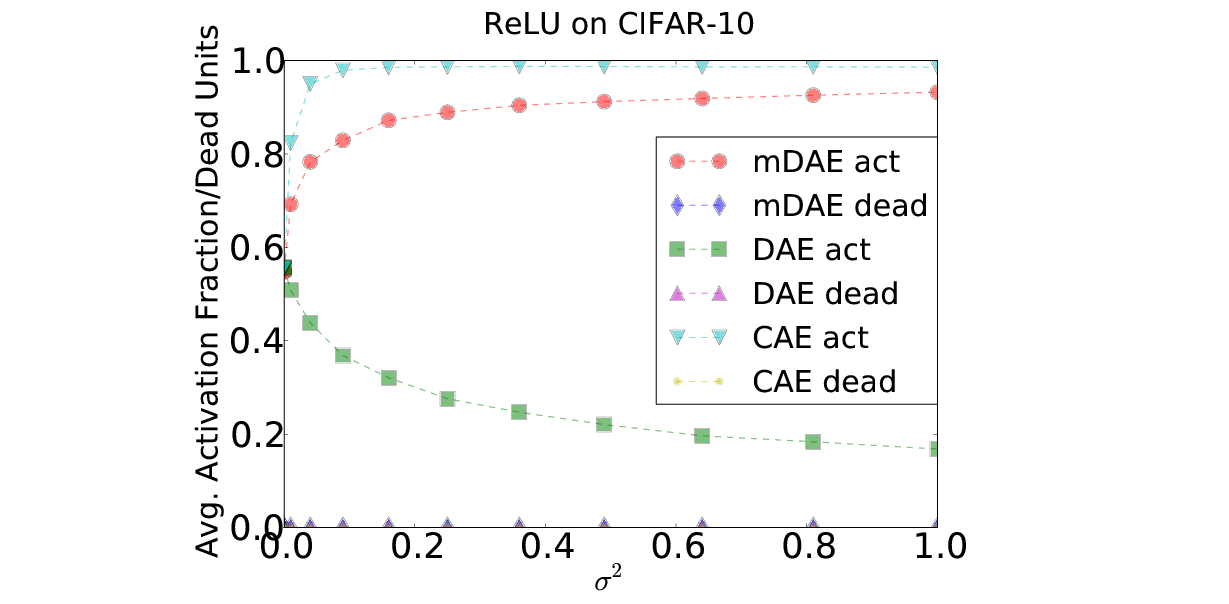}
			\\ 
		\end{tabular}
		\caption{{Trend of average activation fraction \textit{vs.} $ \sigma^{2} $ without weight length constraint using ReLU on MNIST (left) and CIFAR-10 (right).} \label{tab_zero_bias_grad_w_uncons}}
	\end{center}
	\vspace{-20pt}
\end{figure} 

\begin{figure}[b]
	\vspace{-10pt}
	\centering
	\includegraphics[width=0.4\columnwidth,trim=0.25in 0.1in 0.6in 0.38in,clip]{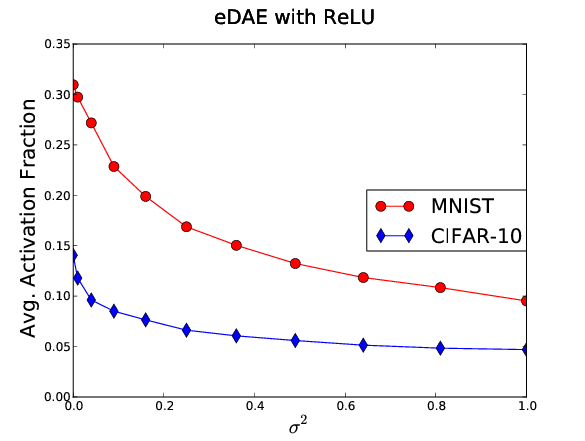}
	\caption{Activation fraction \textit{vs.} $ \sigma^{2} $ for eDAE.\label{fig_bgrad_edae}} 
\end{figure}

\subsubsection{Why is DAE affected by $ \sigma^{2} $ when ReLU has zero bias gradient?}
\label{sec_dae_robust}
The surprising part of the above experiments is that DAE has a stable decreasing sparsity trend (across different values of $ \sigma^{2} $) for ReLU although DAE (similar to CAE, mDAE) has a regularization form given in corollary \ref{cor_ae_reg_form1}. The fact that ReLU practically does not generate bias gradients from this form of regularization brings our attention to an interesting possibility: ReLU is generating the positive bias gradient due to the first order regularization term in DAE. Recall that we marginalize out the first order term in DAE (during Taylor's expansion, see proof of theorem \ref{th_dae}) while taking expectation over all corrupted versions of a training sample. However, the mathematically equivalent objective of DAE obtained by this analytical marginalization is not what we optimize in practice. While optimizing with explicit corruption in a batch-wise manner, we indeed get a non-zero first order term, which does not vanish due to finite sampling (of corrupted versions); thus explaining sparsity for ReLU. We test this hypothesis by optimizing the explicit Taylor's expansion of DAE (eDAE) with only the first order term on MNIST and CIFAR-$ 10 $ using our standard experimental protocols:
\[
\mathcal{J}_{eDAE} = \mathbb{E}_{\mathbf{x}}[\ell(\mathbf{x},f_{d}(f_{e}({\mathbf{x}}))) 
+ (\tilde{\mathbf{x}} - {\mathbf{x}})^{T} \nabla_{\tilde{\mathbf{x}}}\ell]
\]
where $ \tilde{\mathbf{x}} $ is a Gaussian corrupted version of $ {\mathbf{x}} $. The activation fraction \textit{vs.} corruption variance ($\sigma^{2}$) for eDAE is shown in figure \ref{fig_bgrad_edae} which confirms that the first order term contributes towards sparsity. On a more general note, lower order terms (in Taylor's expansion) of highly non-linear functions generally change slower (hence less sensitive) compared to higher order terms. In conclusion we find that explicit corruption may have advantages at times compared to marginalization because it captures the effect of both lower and higher order terms together.

\begin{figure}[t]
	
	\begin{center}
		\begin{tabular}{  c  c  }
			\includegraphics[width=0.45\columnwidth,trim=2.2in 0.1in 3.1in 0.1in,clip]{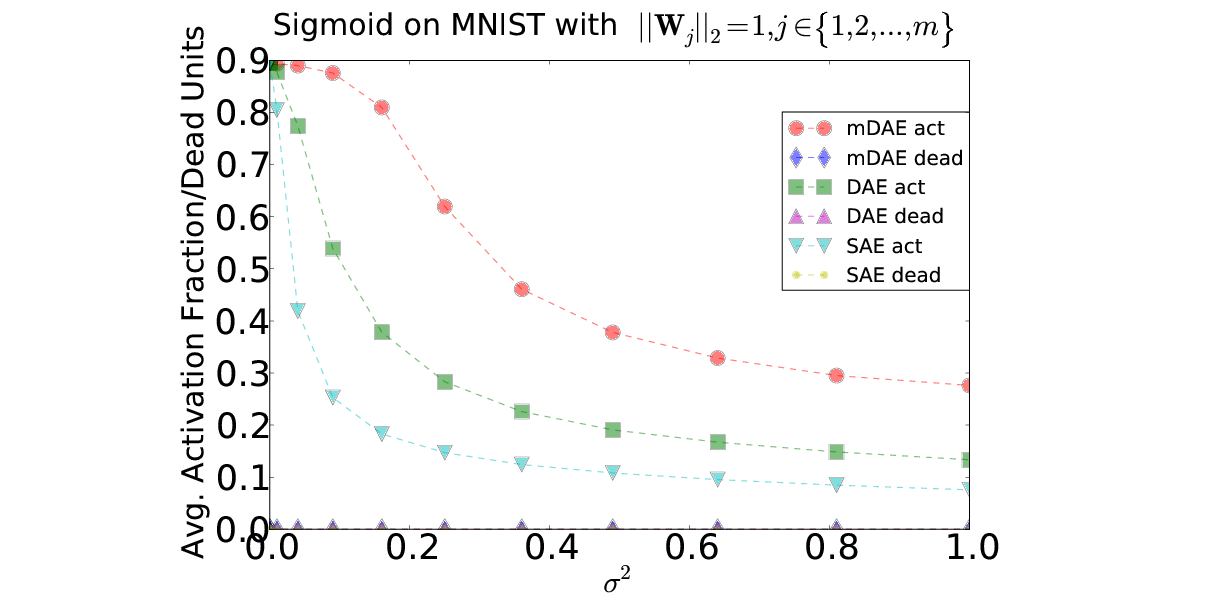}
			& 
			\includegraphics[width=0.45\columnwidth,trim=2.2in 0.1in 3.1in 0.1in,clip]{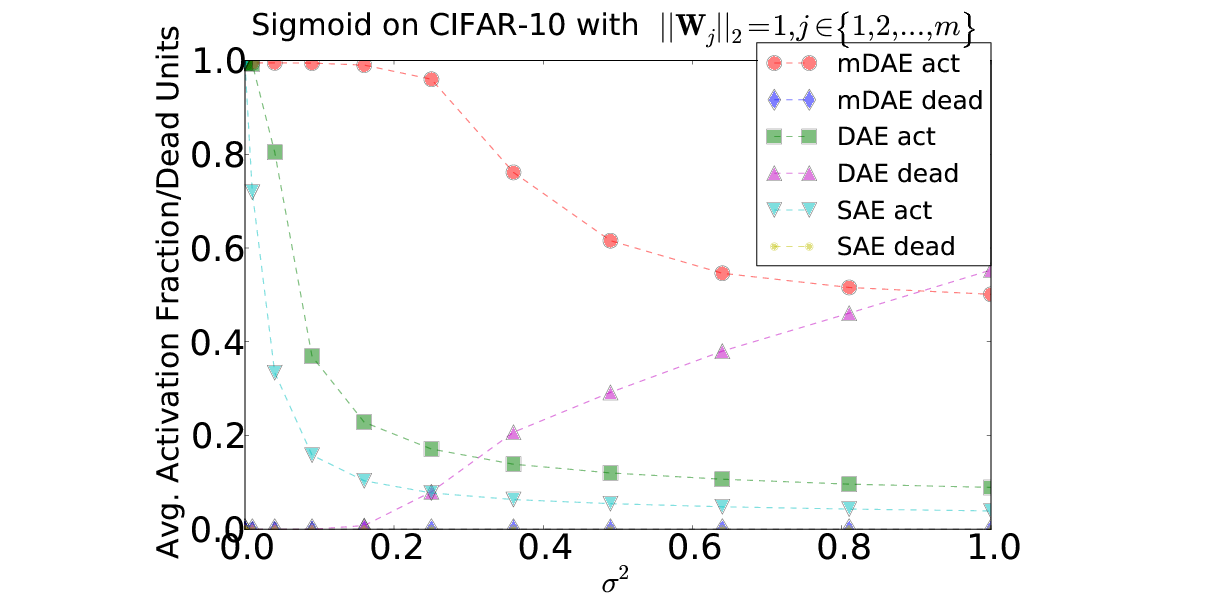}
		\end{tabular}
		\caption{{Trend of average activation fraction \textit{vs.} $ \sigma^{2} $ with weight length constraint using Sigmoid activation MNIST (left) and CIFAR-10 (right).} \label{tab_act_softplus_sigmoid_w1}}
	\end{center}
	\vspace{-20pt}
\end{figure}  

\subsection{Sparsity when Bias Gradient is positive}
As predicted by theorem \ref{th_ae_reg_form}, if the bias gradient is strictly positive ($ \frac{\partial \mathcal{R}}{\partial {b}_{e_{j}}} > 0 $), then increasing the value of $ \sigma^{2} $ should lead to smaller expected pre-activation and thus increasing sparsity. This is specially true when the weight lengths are fixed to some length. This is because term $ \frac{\partial \mathcal{R}}{\partial {b}_{e_{j}}} $ may depend on weight length (depending on the regularization) which is also affected by $ \sigma^{2} $. However, since this effect is hard to predict, sparsity may not always be proportional to $ \sigma^{2} $ for un-constrained weight length.

In order to verify these intuitions,  we first analyse the effect of regularization coefficient ($ \sigma^{2} $) on the sparsity of representations learned by AE models using Sigmoid\footnote{\scriptsize Due to lack of space and because Softplus had trends similar to Sigmoid, we don't show its plots.}\footnote{\scriptsize Although Sigmoid only guarantees sparsity for regularizations in corollary \ref{cor_ae_reg_form2} (eg. SAE), we find it behaves similarly for corollary \ref{cor_ae_reg_form1}(eg. mDAE, CAE).} activation function with weight lengths constrained to one. The plots are shown in figure \ref{tab_act_softplus_sigmoid_w1}. These plots show a stable increasing sparsity trend with increasing regularization coefficient as predicted by our analysis.

Finally, we now analyse the effect of regularization coefficient ($ \sigma^{2} $) on the sparsity of representations learned by AE models using Sigmoid activation function when weight lengths are unconstrained. These plots are shown in figure \ref{tab_act_softplus_sigmoid_w_uncons}. As mentioned above, unconstrained weight length leads to unpredictable behaviour of sparsity with respect to regularization coefficient. This can be seen for mDAE and CAE for both datasets (different trends).

In summary, we again find that weight lengths constrained to have some fixed value lead to better predictive power in terms of sparsity. However in either case, the empirical observations substantiate our claim that sparsity in auto-encoders is dominated by the effect of bias gradient from regularization instead of weight direction. This explains why existing regularized auto-encoders learn sparse representation and the effect of regularization coefficient on sparsity.

\begin{figure}[t]
	
	\begin{center}
		\begin{tabular}{  c  c  }
			\includegraphics[width=0.45\columnwidth,trim=2.4in 0.1in 3.1in 0.1in,clip]{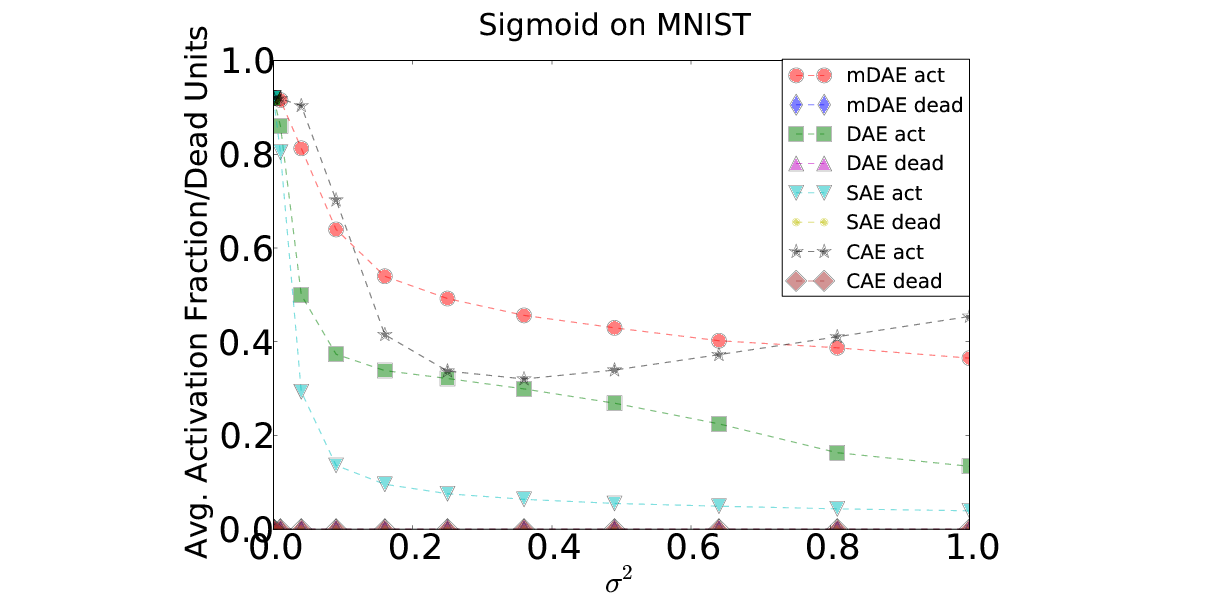}
			& 
			\includegraphics[width=0.45\columnwidth,trim=2.4in 0.1in 3.1in 0.1in,clip]{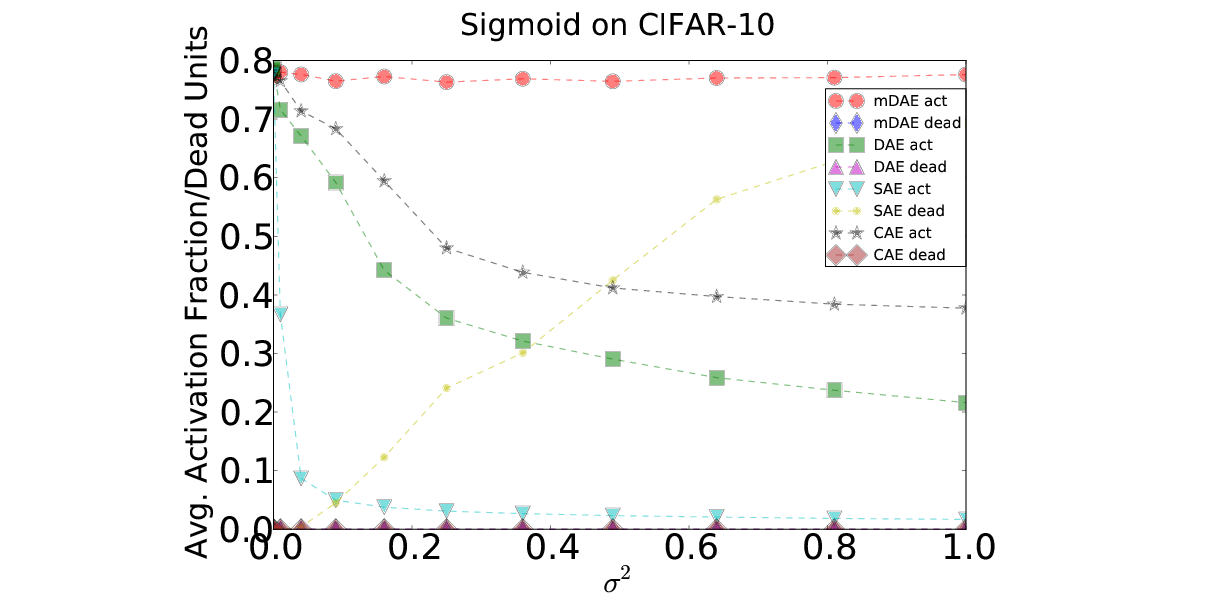}
			\\ 
		\end{tabular}
		\caption{{Trend of average activation fraction \textit{vs.} $ \sigma^{2} $ without weight length constraint using Sigmoid activation MNIST (left) and CIFAR-10 (right).} \label{tab_act_softplus_sigmoid_w_uncons}}
	\end{center}
	\vspace{-20pt}
\end{figure}  

\section{Conclusion and Discussion}
We establish a formal connection between features learned by regularized auto-encoders and sparse representation. Our contribution is multi-fold, we show: a) AE regularizations with positive encoding bias gradient encourage sparsity (theorem \ref{th_ae_reg_form}), while those with zero bias gradient are not affected by regularization coefficient; b) activation functions which are non-decreasing, with negative saturation at zero, encourage sparsity for such regularizations (theorem \ref{th_activation}) and that multiple existing activations have this property (eg. ReLU, Softplus and Sigmoid); c) existing AEs have regularizations of the form suggested in corollary \ref{cor_ae_reg_form2} and \ref{cor_ae_reg_form1}, which not only brings them under a unified framework, but also shows more general forms of regularizations that encourage sparsity.

On the empirical side, a) bias gradient dominates the effect on sparsity of hidden units; specifically sparsity is in general proportional to the regularization coefficient when bias gradient is positive and remains unaffected when it is zero (section \ref{sec_experiments}); b) Constraining the weight vectors during optimization to have fixed length leads to better sparsity and behaviour as predicted by our analysis. Notice this does not restrict the usefulness of the representation leaned by auto-encoders since we are only interested in the filter shapes (weight direction), and not their scale. On the flip side, without length constraint, the behaviour of auto-encoders w.r.t. regularization coefficient becomes unpredictable in some cases. c) explicit corruption (\textit{eg.} DAE) may have advantages over marginalizing it out (\textit{eg.} mDAE, see section \ref{sec_dae_robust}) because it captures both first and second order effects. 

In conclusion, our analysis combined together unifies existing AEs and activation functions by bringing them under a unified framework, but also uncovers more general forms of regularizations and fundamental properties that encourage sparsity in hidden representation. Our analysis also yields new insights into AEs and provides novel tools for analysing existing (and new) regularization/activation functions that help predicting whether the resulting AE learns sparse representations.

{\small
	\bibliographystyle{icml2016}
	\bibliography{ae_sr}
}
\newpage
\onecolumn
\appendix    
\appendixpage
\begin{appendix}
\setcounter{section}{1}
\setcounter{proposition}{0}
\setcounter{theorem}{0}
\setcounter{lemma}{0}
\setcounter{corollary}{0}
\renewcommand{\thesection}{A\arabic{section} }
\textbf{\large \thesection{Supplementary Material}}

\subsection{Supplementary Proofs}
\begin{lemma}
\label{lemma_loss_grad_bound}
If assumption \ref{assump_residual} is true, and encoding activation function $ s_{e}(.) $ has first derivative in $ [0,1] $, then $ {\partial \mathcal{J}_{AE}}/{\partial {b}_{e_{j}}} \in \left[ -2\sigma_{r} \sqrt{n}\lVert \mathbf{W}_{j} \rVert, 2\sigma_{r} \sqrt{n}\lVert \mathbf{W}_{j} \rVert  \right]$.
\end{lemma}
\begin{proof}
For squared loss function $ \mathcal{J}_{AE} $,
\begin{equation}
\frac{\partial \mathcal{J}_{AE}}{\partial {b}_{e_{j}}} = 2 \mathbb{E}_{\mathbf{x}}\left[ \frac{\partial s_{e} (a_{j}) }{ \partial a_{j}} \left( \mathbf{x} - \mathbf{W}^{T}s_{e}(\mathbf{W}\mathbf{x} + \mathbf{b}_{e}) \right)^{T}\mathbf{W}_{j} \right] = 2 \mathbb{E}_{\mathbf{x}}\left[ \frac{\partial s_{e} (a_{j}) }{ \partial a_{j}} \mathbf{r}_{\mathbf{x}}^{T}\mathbf{W}_{j} \right] 
\end{equation}
where $ a_{j} = \mathbf{W}_{j}^{T}\mathbf{x} + b_{j} $. Since $ \frac{\partial s_{e} (a_{j}) }{ \partial a_{j}} \in [0,1] $, 
\begin{equation}
\mathbb{E}_{\mathbf{x}}\left[ \frac{\partial s_{e} (a_{j}) }{ \partial a_{j}} \mathbf{r}_{\mathbf{x}}^{T}\mathbf{W}_{j} \right] \leq \mathbb{E}_{\mathbf{x}}\left[ \frac{\partial s_{e} (a_{j}) }{ \partial a_{j}} \lVert \mathbf{r}_{\mathbf{x}} \rVert  \lVert \mathbf{W}_{j} \rVert \right] \leq \lVert \mathbf{W}_{j} \rVert . \mathbb{E}_{\mathbf{x}}\left[ \lVert \mathbf{r}_{\mathbf{x}} \rVert  \right] 
\end{equation}
Let $ r_{\mathbf{x}} $ denote any one of the elements of $ \mathbf{r}_{\mathbf{x}} $. Since each element of $ \mathbf{r}_{\mathbf{x}} $ is \textit{i.i.d.} from assumption \ref{assump_residual} and $ \mathbf{r}_{\mathbf{x}} \in \mathbb{R}^{n}$, using Jensen's inequality, $ \mathbb{E}_{\mathbf{x}}\left[ \lVert \mathbf{r}_{\mathbf{x}} \rVert_{2}  \right] \leq \sqrt{n \mathbb{E}_{\mathbf{x}}[ r_{\mathbf{x}}^{2}]} = \sqrt{n}\sigma_{r}$. 
Thus,
\begin{equation}
\mathbb{E}_{\mathbf{x}}\left[ \frac{\partial s_{e} (a_{j}) }{ \partial a_{j}} \mathbf{r}_{\mathbf{x}}^{T}\mathbf{W}_{j} \right] \leq  \sqrt{n} \sigma_{r} \lVert \mathbf{W}_{j} \rVert 
\end{equation}
which leads to $ \frac{\partial \mathcal{J}_{AE}}{\partial {b}_{e_{j}}}  \leq  2 \sigma_{r} \sqrt{n} \lVert \mathbf{W}_{j} \rVert$. We can similarly prove in the other direction get the desired bound.
\end{proof}

\begin{theorem}
Let $\{ \mathbf{{W}}^{t} \in \mathbb{R}^{m \times n}, \mathbf{{b}}_{e}^{t} \in \mathbb{R}^{m}  \}$ be the parameters of a regularized auto-encoder ($ \lambda > 0 $)
\begin{equation}
\mathcal{J}_{RAE} = \mathcal{J}_{AE} + \lambda \mathcal{R}({\mathbf{W}},{\mathbf{b}}_{e})
\end{equation}

at training iteration $t$ with regularization term $\mathcal{R}({\mathbf{W}},{\mathbf{b}}_{e})$, activation function $s_{e}(.)$ and define pre-activation ${a}_{j}^{t} = \mathbf{{W}}_{j}^{t}\mathbf{x} + {b}_{e_{j}}^{t}$ (thus ${h}_{j}^{t} = s_{e}({a}_{j}^{t})$). \textbf{If} $\lambda \frac{\partial \mathcal{R}}{\partial {b}_{e_{j}}} > 2 \sigma_{r} \sqrt{n} \lVert \mathbf{W}_{j} \rVert $, where $j \in \{ 1,2, \hdots , m \}$, \textbf{then} updating $\{ \mathbf{{W}}^{t}, \mathbf{{b}}_{e}^{t} \}$ along the negative gradient of $\mathcal{J}_{RAE}$, results in  $\mathbb{E}_{\mathbf{x}}[{a}_{j}^{t+1}] < \mathbb{E}_{\mathbf{x}}[{a}_{j}^{t}]$ \textbf{and} $\var[{{a}_{j}^{t+1}}] = \lVert \mathbf{{W}}_{j}^{t+1} \rVert^{2}$ \textbf{for} all $ t\geq0 $.
\end{theorem}
\begin{proof}
At iteration $t+1$,
\begin{equation}
{a}_{j}^{t+1} = {a}_{j}^{t} - \eta \frac{\partial \mathcal{J}_{RAE}}{\partial \mathbf{\mathbf{W}}_{j}} \mathbf{x} - \eta \frac{\partial \mathcal{J}_{RAE}}{\partial {b}_{e_{j}}}
\end{equation}
for any step size $\eta$. Expanding $ \mathcal{J}_{RAE} $, we get,
\begin{equation}
\label{eq_a_j}
{a}_{j}^{t+1} = {a}_{j}^{t} - \eta \frac{\partial \mathcal{J}_{AE}}{\partial \mathbf{\mathbf{W}}_{j}} \mathbf{x} - \eta \frac{\partial \mathcal{J}_{AE}}{\partial {b}_{e_{j}}} - \eta \lambda \frac{\partial \mathcal{R}}{\partial {\mathbf{W}}_{j}}\mathbf{x} - \eta \lambda \frac{\partial \mathcal{R}}{\partial {b}_{e_{j}}}
\end{equation}
Thus taking expectation over $ \mathbf{x} $ on both sides we get,
\begin{equation}
\mathbb{E}_{\mathbf{x}} \left[ {a}_{j}^{t+1} \right] = \mathbb{E}_{\mathbf{x}} \left[ {a}_{j}^{t} \right] - \eta \frac{\partial \mathcal{J}_{AE}}{\partial {b}_{e_{j}}} - \eta\lambda  \frac{\partial \mathcal{R}}{\partial {b}_{e_{j}}}
\end{equation}
Notice the terms containing $ \frac{\partial \mathcal{J}_{AE}}{\partial \mathbf{\mathbf{W}}_{j}} $ and $ \frac{\partial \mathcal{R}}{\partial {\mathbf{W}}_{j}} $ in equation \ref{eq_a_j} disappear because both terms are already a function of expectation over $ \mathbf{x} $ (see various auto-encoder regularizations) when we deal with expected cost function. Thus these terms are linear in $ \mathbf{x} $ and hence taking an expectation results in $ 0 $.

From lemma \ref{lemma_loss_grad_bound}, $ \frac{\partial \mathcal{J}_{AE}}{\partial {b}_{e_{j}}}  \geq  -2 \epsilon \sqrt{n} \lVert \mathbf{W}_{j} \rVert $, thus if $\lambda \frac{\partial \mathcal{R}}{\partial {b}_{e_{j}}} > 2 \sigma_{r} \sqrt{n} \lVert \mathbf{W}_{j} \rVert$, then $\mathbb{E}_{\mathbf{x}}[{a}_{j}^{t+1}] < \mathbb{E}_{\mathbf{x}}[{a}_{j}^{t}]$.

Finally, $\var[{a}_{j}^{t+1}] = \mathbb{E}_{\mathbf{x}} [ {a}_{j}^{t+1} -  \mathbb{E}_{\mathbf{x}} [{a}_{j}^{t+1}] ]^{2} = \mathbb{E}_{\mathbf{x}} [\mathbf{{W}}_{j}^{t+1}\mathbf{x}]^{2} =  \lVert \mathbf{{W}}_{j}^{t+1} \rVert^{2}$
\end{proof}

\vspace{12pt}
\begin{corollary}
\textbf{If} $ s_{e}$ is a non-decreasing activation function with first derivative in $ [0,1] $ \textbf{and} $\mathcal{R}= \sum_{j=1}^{m} f(\mathbb{E}_{\mathbf{x}}[{h}_{j}])$ for any monotonically increasing function $ f(.) $, \textbf{then} $ \exists \lambda>0 $ such that updating $\{ \mathbf{{W}}^{t}, \mathbf{{b}}_{e}^{t} \}$ along the negative gradient of $ \mathcal{J}_{RAE} $ results in  $\mathbb{E}_{\mathbf{x}}[{a}_{j}^{t+1}] \leq \mathbb{E}_{\mathbf{x}}[{a}_{j}^{t}]$ \textbf{and} $\var[{{a}_{j}^{t+1}}] = \lVert \mathbf{{W}}_{j}^{t+1} \rVert^{2}$ \textbf{for} all $ t\geq0 $.
\begin{proof}
We need one additional argument other than theorem \ref{th_ae_reg_form}. $\frac{\partial \mathcal{R}}{\partial {b}_{e_{j}}} = \frac{\partial f(\mathbb{E}_{\mathbf{x}}[{h}_{j}])}{\partial \mathbb{E}_{\mathbf{x}}[{h}_{j}]} \mathbb{E}_{\mathbf{x}} \left[\frac{\partial {h}_{j}}{\partial {a}_{j}} \right]$. Since both $ s_{e}(.) $ and $ f(.) $ are non-decreasing functions, $\frac{\partial \mathcal{R}}{\partial {b}_{e_{j}}} \geq 0$ in all cases.
\end{proof}
\end{corollary}

\vspace{12pt}
\begin{corollary}
\textbf{If} $ s_{e} $ is a non-decreasing convex activation function with first derivative in $ [0,1] $ \textbf{and} $\mathcal{R}= \mathbb{E}_{\mathbf{x}} \left[  \sum_{j=1}^{m} \left( \left(\frac{\partial {h}_{j}}{\partial {a}_{j}}\right)^{q}\lVert \mathbf{{W}}_{j}^{t}\rVert_{2}^{p} \right) \right]$, $q \in \mathbb{N}$ , $p \in \mathbb{W}$, \textbf{then} $ \exists \lambda>0 $ such that updating $\{ \mathbf{{W}}^{t}, \mathbf{{b}}_{e}^{t} \}$ along the negative gradient of $ \mathcal{J}_{RAE} $, results in  $\mathbb{E}_{\mathbf{x}}[{a}_{j}^{t+1}] \leq \mathbb{E}_{\mathbf{x}}[{a}_{j}^{t}]$ \textbf{and} $\var[{{a}_{j}^{t+1}}] =  \lVert \mathbf{{W}}_{j}^{t+1} \rVert^{2}$ \textbf{for} all $ t\geq0 $.
\begin{proof}
We need one additional argument other than theorem \ref{th_ae_reg_form}. $\frac{\partial \mathcal{R}}{\partial {b}_{e_{j}}} = \mathbb{E}_{\mathbf{x}} \left[ q\left(\frac{\partial {h}_{j}}{\partial {a}_{j}}\right)^{q-1} \frac{\partial^{2} {h}_{j}}{\partial {a}_{j}^{2}} \frac{\partial {a}_{j}}{\partial {b}_{e_{j}}} \lVert \mathbf{{W}}_{j}^{t}\rVert_{2}^{p} \right]$. Since $ s_{e}(.) $ is a non-decreasing convex function, both $\frac{\partial^{2} s_{e}({a}_{j})}{\partial {a}_{j}^{2}} \geq 0$ and $\frac{\partial s_{e}({a}_{j})}{\partial {a}_{j}} \geq 0$ $ \forall a_{j} \in \mathbb{R} $. Finally, $ \frac{\partial {a}_{j}}{\partial {b}_{e_{j}}}=1 $ by definition. Thus $\frac{\partial \mathcal{R}}{\partial {b}_{e_{j}}} \geq 0$ in all cases.
\end{proof}
\end{corollary}

\begin{theorem} Let $ p_{j}^{t}$ denote a lower bound of $
	\prob({h}_{j}^{t}\leq \delta_{\min}) $ at iteration $ t $ and $ s_{e}(.) $ be a non-decreasing function with first derivative in $ [0,1] $. \textbf{If} $ \lVert \mathbf{W}_{j}^{t} \rVert_{2} $ is upper bounded independent of $ \lambda $ \textbf{then} $ \exists S \subseteq \mathbb{R}^{+} $ \textbf{and} $ \exists T_{\min},T_{\max} \in \mathbb{N}$ \textbf{such that} $ p_{j}^{t+1}\geq p_{j}^{t} $ $ \forall \lambda \in S $, $T_{\min} \leq t \leq T_{\max} $.
\end{theorem}
\begin{proof}
	From theorem \ref{th_ae_reg_form}, $
		\mathbb{E}[{a}_{j}^{t+1}] < \mathbb{E}[{a}_{j}^{t}] $ $ \forall
		t\geq 0 $. Define $ {a}_{\min} $ such that $\delta_{\min} = \max_{a_{\min}} s_{e}({a}_{\min}) $. Thus $ \exists T_{\min} \in \mathbb{N} $, such that $ \forall
		t\geq T_{\min} $, $ \mathbb{E}[{a}_{j}^{t}] <{a}_{\min}$.
		Then in the case of non-decreasing activation functions, using Chebyshev's bound,
		\begin{equation}
		\begin{split}
		\prob({h}_{j}^{t}\leq \delta_{\min})  =
		\prob({a}_{j}^{t}\leq {a}_{\min}) \geq
		\prob(|{a}_{j}^{t} - \mathbb{E}[{a}_{j}^{t}]| \leq
		{a}_{\min} - \mathbb{E}[{a}_{j}^{t}]) \\
		\geq 1- \frac{\var[{a}_{j}^{t}]}{({a}_{\min} - \mathbb{E}[{a}_{j}^{t}])^{2}}
		\end{split}
		\end{equation}
		Thus $ p_{j}^{t} := 1 -
		\frac{\var[{a}_{j}^{t}]}{({a}_{\min} - \mathbb{E}[{a}_{j}^{t}])^{2}} $
		lower bounds $ \prob({h}_{j}^{t}\leq \delta_{\min}) $ $ \forall t\geq T_{\min} $. Now consider the difference
		\begin{equation}
		D(t) := \frac{\var[{a}_{j}^{t+1}]}{({a}_{\min} - \mathbb{E}[{a}_{j}^{t+1}])^{2}} - \frac{\var[{a}_{j}^{t}]}{({a}_{\min} - \mathbb{E}[{a}_{j}^{t}])^{2}}
		\end{equation}
		and recall that
		\begin{equation}
		\mathbb{E}_{\mathbf{x}} \left[ {a}_{j}^{t+1} \right] = \mathbb{E}_{\mathbf{x}} \left[ {a}_{j}^{t} \right]  - \eta \frac{\partial \mathcal{J}_{AE}}{\partial {b}_{e_{j}}} - \eta\lambda  \frac{\partial \mathcal{R}}{\partial {b}_{e_{j}}}
		\end{equation}
		where both the step size $ \eta $ and $\frac{\partial \mathcal{R}}{\partial {b}_{e_{j}}}$ are positive and  $ {\partial \mathcal{J}_{AE}}/{\partial {b}_{e_{j}}} \in \left[ -2\sigma_{r} \sqrt{n}\lVert \mathbf{W}_{j} \rVert, 2\sigma_{r} \sqrt{n}\lVert \mathbf{W}_{j} \rVert  \right]$. Thus, since $ \var[a_{j}] = \lVert \mathbf{W}_{j}^{t} \rVert^{2}$, we can always choose a fixed $ S \subseteq \mathbb{R}^{+} $ such that $ D(t)\leq 0 $ $ \forall \lambda \in S $ and $T_{\min} \leq t \leq T_{\max} $.
	\end{proof}

\vspace{12pt}
\begin{theorem}
Let $\{ \mathbf{{W}},\mathbf{{b}}_{e}\}$ represent the parameters of a DAE with squared loss, linear decoding, and i.i.d. Gaussian corruption with zero mean and $\sigma^{2}$ variance, at any point of training over data sampled from distribution $\mathcal{D}$. Let ${a}_{j} := \mathbf{{W}}_{j}\mathbf{x} + {b}_{e_{j}}$ so that ${h}_{j} = s_{e}({a}_{j})$ corresponding to sample $\mathbf{x} \sim \mathcal{D}$.
Then,
\begin{equation}
\begin{split}
\mathcal{J}_{DAE} = \mathcal{J}_{AE} + 
\sigma^{2} \mathbb{E}_{\mathbf{x}} \left[  \sum_{j=1}^{m} \left( \left(\frac{\partial {h}_{j}}{\partial {a}_{j}}\right)^{2}\lVert \mathbf{{W}}_{j} \rVert_{2}^{4} \right) + 
\sum_{\substack{j,k=1 \\ j\neq k}}^{m}\left( \frac{\partial {h}_{j}}{\partial {a}_{j}} \frac{\partial {h}_{k}}{\partial {a}_{k}} (\mathbf{{W}}_{j}^{T} \mathbf{{W}}_{k})^{2} \right)    \right. \\
\left. +  \sum_{i=1}^{n}\left(  ({\mathbf{b}_{d} + \mathbf{W}}^{T}{\mathbf{h}} - \mathbf{x})^{T}{\mathbf{W}}^{T} \left( \frac{\partial^{2} {\mathbf{h}} }{\partial {\mathbf{a}}^{2}} \odot {\mathbf{W}}^{i} \odot{\mathbf{W}}^{i} \right) \right) \right]  + o(\sigma^{2})
\end{split}
\end{equation}
where $ \frac{\partial^{2} {\mathbf{h}} }{\partial {\mathbf{a}}^{2}}  \in \mathbb{R}^{m}$ is the element-wise $ 2^{nd} $ derivative of $ {\mathbf{h}} $ \textit{w.r.t.} $ {\mathbf{a}} $ and $ \odot $ is element-wise product.

\begin{proof}
Using $2^{nd}$ order Taylor's expansion of the loss function, we get

\begin{equation}
\begin{split}
\ell(\mathbf{x},f_{d}(f_{e}(\tilde{\mathbf{x}}))) = \ell(\mathbf{x},f_{d}(f_{e}(\mu_{\mathbf{x}}))) 
+ (\tilde{\mathbf{x}} - \mu_{\mathbf{x}})^{T} \nabla_{\tilde{\mathbf{x}}}\ell + \frac{1}{2}(\tilde{\mathbf{x}} - \mu_{\mathbf{x}})^{T} \nabla_{\tilde{\mathbf{x}}}^{2}\ell \mspace{5mu}(\tilde{\mathbf{x}} - \mu_{\mathbf{x}}) + o(\sigma^{2})
\end{split}
\end{equation}
where $ \mu_{\mathbf{x}}  = \mathbf{x}$. since we assume zero mean Gaussian noise. Thus taking the expectation of this approximation over noise yields
\begin{equation}
\label{eq_exp_loss_approx}
\mathbb{E}[\ell(\mathbf{x},f_{d}(f_{e}(\tilde{\mathbf{x}})))] = \mathbb{E}[\ell(\mathbf{x},f_{d}(f_{e}(\mu_{\mathbf{x}})))] + \frac{1}{2}tr(\Sigma_{\mathbf{x}}\nabla_{\tilde{\mathbf{x}}}^{2}\ell) + o(\sigma^{2})
\end{equation}
where $\Sigma_{\mathbf{x}} := \mathbb{E}[(\tilde{\mathbf{x}} - \mu_{\mathbf{x}})(\tilde{\mathbf{x}} - \mu_{\mathbf{x}})^{T}]$. Since the corruption is $i.i.d.$, assume the covariance $\Sigma_{\mathbf{x}} = \sigma^{2} \mathbf{I}$, where $\mathbf{I}$ is the identity matrix. 

Taking expectation over $ \mathbf{x} $, we can rewrite equation (\ref{eq_exp_loss_approx}) as
\begin{equation}
\label{eq_exp_loss_approx2}
\mathcal{J}_{DAE} = \mathcal{J}_{AE} + \mathbb{E}_{\mathbf{x}} \left[ \frac{1}{2}\sigma^{2} \sum_{i=1}^{n} \frac{\partial^{2} \ell}{\partial \tilde{{x}}_{i}^{2}} \right] + o(\sigma^{2})
\end{equation}
Expanding the second order term in the above equation, we get
\begin{equation}
\label{eq_yet_another_approx}
\frac{\partial^{2} \ell}{\partial \tilde{{x}}_{i}^{2}} = \frac{\partial \mathbf{{h}}}{\partial \tilde{{x}}_{i}}^{T} \frac{\partial^{2} \ell}{\partial {\mathbf{{h}}^{2}}} \frac{\partial \mathbf{{h}}}{\partial \tilde{{x}}_{i}} + \frac{\partial \ell}{\partial {\mathbf{{h}}}}^{T} \frac{\partial^{2} \mathbf{{h}}}{\partial \tilde{{x}}_{i}^{2}}
\end{equation}
For linear decoding and squared loss, 
\begin{equation}
\frac{\partial \ell}{\partial {\mathbf{{h}}}}^{T} \frac{\partial^{2} \mathbf{{h}}}{\partial \tilde{{x}}_{i}^{2}} = \sum_{i=1}^{n}\left(  ({\mathbf{b}_{d} +\mathbf{W}}^{T}{\mathbf{h}} - \mathbf{x})^{T}{\mathbf{W}}^{T} \left( \frac{\partial^{2} {\mathbf{h}} }{\partial {\mathbf{a}}^{2}} \odot {\mathbf{W}}^{i} \odot{\mathbf{W}}^{i} \right) \right)
\end{equation}
where $ \frac{\partial^{2} {\mathbf{h}} }{\partial {\mathbf{a}}^{2}}  \in \mathbb{R}^{m}$ is the element-wise $ 2^{nd} $ derivative of $ {\mathbf{h}} $ \textit{w.r.t.} $ {\mathbf{a}} $, $ \odot $ represents element-wise product and $ \mathbf{W}^{i} $ denotes the $ i^{th} $ column of $\mathbf{W}$. Let vector $ {\mathbf{d}_{{\mathbf{h}}}} \in \mathbb{R}^{m}$ be defined such that $ {d}_{{h}_{j}} = \frac{\partial {h}_{j}}{\partial {a}_{j}} $ $ \forall j \in \{ 1,2, \hdots ,m \} $. Then,
\begin{equation}
\label{eq_second_order_term}
\sum_{i=1}^{n} \frac{\partial \mathbf{{h}}}{\partial \tilde{{x}}_{i}}^{T} \frac{\partial^{2} \ell}{\partial {\mathbf{{h}}^{2}}} \frac{\partial \mathbf{{h}}}{\partial \tilde{{x}}_{i}} = 2  \sum_{j=1}^{n} \sum_{k=1}^{n} \left( (\mathbf{d}_{{\mathbf{h}}} \odot (\mathbf{{W}})^{j} )^{T}(\mathbf{{W}})^{k} \right)^{2}
\end{equation}

where $(\mathbf{{W}})^{j}$ represents the $j^{th}$ column of $\mathbf{W}$ and $\odot$ denotes element-wise product. Let $\mathbf{D}_{\mathbf{{h}}} = \diag(\mathbf{d}_{{\mathbf{h}}})$. Then,
\begin{equation}
\label{eq_second_order_term2}
\sum_{j=1}^{n} \sum_{k=1}^{n} \left( (\mathbf{d}_{{\mathbf{h}}} \odot (\mathbf{{W}})^{j} )^{T}(\mathbf{{W}})^{k} \right)^{2} = \lVert (\mathbf{D}_{\mathbf{{h}}}\mathbf{{W}})^{T}\mathbf{{W}} \rVert_{F}^{2}
\end{equation}
Finally, using the cyclic property of trace operator, we get, $\lVert (\mathbf{D}_{\mathbf{{h}}}\mathbf{{W}})^{T}\mathbf{{W}} \rVert_{F}^{2} = \tr(\mathbf{{W}}^{T}\mathbf{D}_{\mathbf{{h}}}\mathbf{{W}}\mathbf{{W}}^{T}\mathbf{D}_{\mathbf{{h}}}\mathbf{{W}}) = \tr(\mathbf{D}_{\mathbf{{h}}}\mathbf{{W}}\mathbf{{W}}^{T}\mathbf{D}_{\mathbf{{h}}}\mathbf{{W}}\mathbf{{W}}^{T})$. Thus DAE objective becomes,
\begin{equation}
\begin{split}
\mathcal{J}_{DAE} = \mathcal{J}_{AE}
+ \sigma^{2} \mathbb{E}_{\mathbf{x}} \left[\tr(\mathbf{D}_{\mathbf{{h}}}\mathbf{{W}}\mathbf{{W}}^{T}\mathbf{D}_{\mathbf{{h}}}\mathbf{{W}}\mathbf{{W}}^{T}) + \right. \\
\left. \sum_{i=1}^{n}\left(  ({\mathbf{b}_{d} +\mathbf{W}}^{T}{\mathbf{h}} - \mathbf{x})^{T}{\mathbf{W}}^{T} \left( \frac{\partial^{2} {\mathbf{h}} }{\partial {\mathbf{a}}^{2}} \odot {\mathbf{W}}^{i} \odot{\mathbf{W}}^{i} \right) \right) \right] + o(\sigma^{2})
\end{split}
\end{equation}

Upon expansion of the second term above, we get the final form.
\end{proof}
\end{theorem}

\begin{remark}
	Let $\{ \mathbf{{W}} \in \mathbb{R}^{m \times n},\mathbf{{b}}_{e} \in \mathbb{R}^{m} \}$ represent the parameters of a Marginalized De-noising Auto-Encoder (mDAE) with $ s_{e}(.) $ activation function, linear decoding, squared loss and $\sigma_{\mathbf{x}i}^{2} = \lambda$ $\forall i \in \{ 1, \hdots ,n \}$, at any point of training over data sampled from some distribution $\mathcal{D}$. Let ${a}_{j} := \mathbf{{W}}_{j}\mathbf{x} + {b}_{e_{j}}$ so that ${h}_{j} = s_{e}({a}_{j})$ corresponding to sample $\mathbf{x} \sim \mathcal{D}$. Then,
	\begin{equation}
	\label{eq_mdae_equiv}
	\mathcal{J}_{mDAE} = \mathcal{J}_{AE} + \lambda \mathbb{E}_{\mathbf{x}} \left[  \sum_{j=1}^{m} \left( \left(\frac{\partial {h}_{j}}{\partial {a}_{j}}\right)^{2} \lVert \mathbf{{W}}_{j} \rVert_{2}^{4} \right) \right]
	\end{equation}
	
	\begin{proof}
		For linear decoding and squared loss, $\frac{\partial^{2} \ell}{\partial {{{h}_{j}}^{2}}} = 2 \lVert \mathbf{{W}}_{j} \rVert_{2}^{2}$ and $\frac{\partial {{h}_{j}}}{\partial \tilde{\mathbf{x}}_{i}} = \frac{\partial {h}_{j}}{\partial {a}_{j}}W_{ji}$. Thus
		\begin{equation}
		\begin{split}
		\frac{1}{2} \sum_{i=1}^{n} \sigma_{\mathbf{x}i}^{2} \sum_{j=1}^{m}  \frac{\partial^{2} \ell}{\partial {{{h}_{j}}^{2}}} \left( \frac{\partial {{h}_{j}}}{\partial \tilde{\mathbf{x}}_{i}} \right)^{2} = \sum_{i=1}^{n} \lambda \sum_{j=1}^{m} \lVert \mathbf{{W}}_{j} \rVert_{2}^{2} \left( \frac{\partial {h}_{j}}{\partial {a}_{j}}{W}_{ji} \right)^{2} \\
		=  \lambda \sum_{j=1}^{m} \lVert \mathbf{{W}}_{j} \rVert_{2}^{2}  \left( \frac{\partial {h}_{j}}{\partial {a}_{j}} \right)^{2} \sum_{i=1}^{n}{W}^{2}_{ji} = \lambda \sum_{j=1}^{m}  \left( \frac{\partial {h}_{j}}{\partial {a}_{j}} \right)^{2} \lVert \mathbf{{W}}_{j} \rVert_{2}^{4} \mspace{20mu}
		\end{split}
		\end{equation}
	\end{proof}
\end{remark}
\end{appendix}

\end{document}